\documentclass[preprint,5p]{elsarticle}

\usepackage{lineno}
\usepackage{hyperref}
\modulolinenumbers[5]
\usepackage{prletters}

\usepackage{times}
\usepackage{soul}
\usepackage{url}
\usepackage{graphicx}
\usepackage{amsmath}
\usepackage{amsthm}
\usepackage{booktabs}
\usepackage{algorithm}
\usepackage{algorithmic}
\usepackage{subfigure}
\usepackage{amssymb}
\usepackage{multirow}
\usepackage{color}
\usepackage{bm}

\newtheorem{mydef}{\bf Definition}[section]
\newtheorem{thm}{\bf Theorem}[section]

\newcommand{\expect}{\mathbb{E}}

\definecolor{Brown}{rgb}{0.8, 0.3, 0.25}

\journal{Journal of \LaTeX\ Templates}

\definecolor{mydarkblue}{rgb}{0,0.08,0.45}
\definecolor{mydarkgreen}{RGB}{0, 139, 69}
\hypersetup{
	colorlinks=true,
	urlcolor=magenta,
	citecolor=mydarkblue,
}

\begin{document}

\begin{frontmatter}

\title{Consistent Attack: Universal Adversarial Perturbation on Embodied Vision Navigation}



\author[1]{Chengyang Ying\fnref{fn1}}\ead{ycy21@mails.tsinghua.edu.cn}
\author[1]{You Qiaoben\fnref{fn1}}\ead{qby17@mails.tsinghua.edu.cn}
\author[1]{Xinning Zhou\fnref{fn1}}\ead{zxn21@mails.tsinghua.edu.cn}
\author[1,2]{Hang Su\corref{cor1}}\ead{suhangss@mail.tsinghua.edu.cn}
\author[3]{Wenbo Ding}\ead{dingwenbo@saicmotor.com}
\author[3]{Jianyong Ai}\ead{aijianyong@saicmotor.com}

\fntext[fn1]{Equal Contribution}
\cortext[cor1]{Corresponding author}


\address[1]{Department of Computer Science and Technology, Beijing National Research Center for Information Science and Technology, Tsinghua-Bosch Joint Center for Machine Learning, Institute for Artificial Intelligence,
Tsinghua University, Beijing 100084, China}
\address[2]{Peng Cheng Laboratory, Shenzhen, Guangdong, 518055, China}
\address[3]{SAIC AI Lab}





\begin{abstract}
Embodied agents in vision navigation coupled with deep neural networks have attracted increasing attention. However, deep neural networks have been shown vulnerable to malicious adversarial noises, which may potentially cause catastrophic failures in Embodied Vision Navigation.
Among different adversarial noises, universal adversarial perturbations (UAP), i.e., a constant image-agnostic perturbation applied on every input frame of the agent, play a critical role in Embodied Vision Navigation since they are computation-efficient and application-practical during the attack.
However, existing UAP methods ignore the system dynamics of Embodied Vision Navigation and might be sub-optimal.
In order to extend UAP to the sequential decision setting, we formulate the disturbed environment under the universal noise $\delta$, as a $\delta$-disturbed Markov Decision Process ($\delta$-MDP). Based on the formulation, we analyze the properties of $\delta$-MDP and propose two novel Consistent Attack methods, named Reward UAP and Trajectory UAP, for attacking Embodied agents, which consider the dynamic of the MDP and calculate universal noises by estimating the disturbed distribution and the disturbed Q function.
For various victim models, our Consistent Attack can cause a significant drop in their performance in the PointGoal task in Habitat with different datasets and different scenes. Extensive experimental results indicate that there exist serious potential risks for applying Embodied Vision Navigation methods to the real world.
\end{abstract}


\end{frontmatter}


\begin{figure}[t]
    \centering
    \includegraphics[scale=0.28]{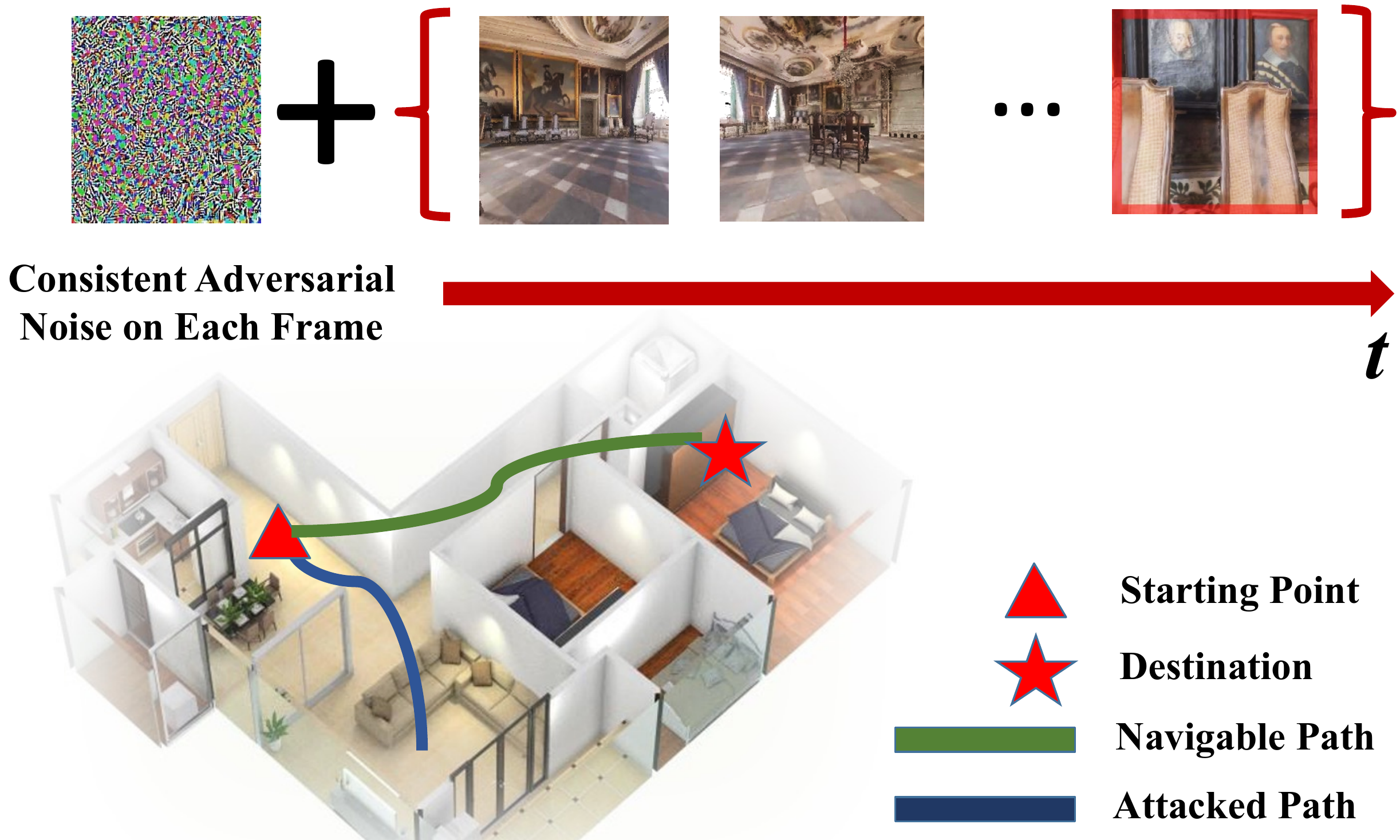}
    \vspace{-1em}
    \caption{An illustration of universal adversarial perturbation on observations of an agent in Embodied Vision Navigation. In this task, the agent needs to navigate from the \textbf{red triangle} to the \textbf{red star}. The \textbf{green curve} is the navigable path. At each timestep, the adversary adds a consistent noise to the input frame. Finally, the adversary misleads the agent to take the \textbf{blue curve}.}
    \label{fig:problem_setting}
    \vspace{-1em}
\end{figure}

\section{Introduction}
\label{sec-intro}
Embodied Vision Navigation, aiming at reaching a given point or object with sequential vision inputs, is a core component of Embodied Artificial Intelligence (Embodied AI) and has gradually attracted the attention of researchers~\cite{hermann2017grounded,anderson2018evaluation}.
Coupled with deep neural networks, Embodied Vision Navigation agents combined with techniques of vision signal processing and sequential decision have made sustained progress~\cite{anderson2018evaluation,chaplot2020object,ye2021hierarchical}. 
Recently, the vulnerability of deep neural networks has been well studied~\cite{biggio2013evasion,szegedy2013intriguing,goodfellow2014explaining}, while there is little attention focusing on the robustness of Embodied Vision Navigation agents. Therefore, it is important to address this problem before we can deploy Embodied Vision Navigation agents to the real world.

In the setting of Embodied Vision Navigation, it is impractical to hack the system and generate specific attacks at every timestep, and thus adding a universal adversarial perturbation (UAP)~\cite{moosavi2017universal} to every observation is a more reasonable setting than calculating observation-wise noises. Moreover, UAP is more feasible to be deployed, e.g. sticking a patch on the sensor can generate universal perturbation.
Also, current adversarial attack methods usually require complex computation to generate noises and are thus slow and expensive to be deployed~\cite{Huang2017Adversarial,zhang2020robust}. Therefore, as illustrated in Fig.~\ref{fig:problem_setting}, our goal is to study UAP, which is one of the most cost-effective adversarial attacks, against Embodied agents. We aim to find the optimal universal noise $\delta^*$ to add to every observation, which can mislead the victim and prevent it from reaching the goal.

Although there is an array of UAP methods~\cite{moosavi2017universal,khrulkov2018art,zhang2020understanding} that have achieved promising results in image classification, the potential risk of UAP in Embodied Vision Navigation has been rarely investigated. Moreover, existing UAP methods consider the setting that all samples are i.i.d sampled from the same distribution and are independent of each other, which no longer holds since Embodied Vision Navigation is a Markov Decision Process (MDP) and the sequential frames are related to each other. Consequently, directly applying previous UAP methods to Embodied Vision Navigation is not the optimal choice.


To handle this gap, in this paper, we first introduce $\delta$-Markov Decision Process ($\delta$-MDP) to formulate the persistent effect of the Universal Adversarial Perturbation, based on which we propose Disturbed Policy Gradient Theorem (Theorem~\ref{thm-2}) to calculate the gradient of the disturbance $\delta$ via the disturbed Q function. According to the result, we need to sample state-action pairs from the disturbed distribution and estimate the disturbed Q function for calculating the gradient. Furthermore, we propose two novel Consistent Attack methods named Reward UAP and Trajectory UAP for calculating the optimal universal adversarial perturbation of Embodied Vision Navigation agents. In Reward UAP, we first propose a multi-step optimization framework to ensure trajectories in every step are sampled from the disturbed policy, and then we use two different ways to estimate the disturbed Q function. In Trajectory UAP, we consider Embodied Vision Navigation as a goal-condition problem and just use the goal signal to optimize our universal adversarial perturbation, which is a more practical assumption since the adversary can easily know whether the agent reaches the goal rather than the actual reward signal of the agent.


Extensive experimental results suggest that our Consistent Attack can significantly reduce the performance of existing Embodied Vision Navigation methods in Habitat~\cite{savva2019habitat}, which shows the effectiveness of our method and the serious vulnerability of existing Embodied Vision Navigation methods. 


\section{Related Work}

\label{sec-related}
\subsection{Embodied Vision Navigation}




Since the agent in Embodied Vision Navigation needs to actively interact with the environment~\cite{beattie2016deepmind,kolve2017ai2,savva2019habitat}, its safety requirement is more important compared with traditional AI tasks, e.g., image classification, sentence generation, and audio recognition. For example, in the PointGoal navigation task~\cite{anderson2018evaluation} in Habitat, the agent tries to find a path to reach the destination, while it might suffer from catastrophic failure when collides along the way~\cite{savva2019habitat}. 
In this paper, we mainly focus on the security of the agent~\cite{savva2019habitat} in the PointGoal task, which is a common and popular baseline in Embodied AI.

There are many works considering the robustness of models in different computer vision tasks~\cite{hu2022adversarial,Hua2019DefendingAA,yang2021multiple}, like face recognition~\cite{Hu2022ProtectingFP,Yang2020TowardsFE} and face forgery detection~\cite{Gandhi2020AdversarialPF,Neekhara2020AdversarialDE,Zhang2021DetectingDV}. However, a major difference between these tasks and Embodied Vision Navigation is whether the inputs are independent and identically distributed.

\subsection{Adversarial Attacks}
\label{sec-2-2}

As a special adversarial attack method in image classification, universal adversarial perturbation (UAP) aims to fool the classifier for most images via a constant noise~\cite{zhang2021survey,moosavi2017universal,khrulkov2018art,zhang2020understanding}. However, existing UAP methods are limited to datasets where images are mutually independent and are thus sub-optimal for Embodied Vision Navigation since it fails to explore the dependency between different images in the same trajectory.


After~\cite{Huang2017Adversarial} first proposed to use FGSM-based adversarial noises to attack policies with deep neural networks, there are an array of works focusing on the vulnerability of deep reinforcement learning (DRL) agents.~\cite{zhang2020robust} proposed a framework named state-adversarial Markov decision process (SA-MDP) for analyzing the adversarial disturbance on state observation in reinforcement learning (RL). 
Recently,~\cite{qiaoben2021understanding} summarized existing adversarial attack methods on observations in RL and proposed a novel two-stage adversarial attack method. However, these adversarial attack methods generate different disturbances at each timestep, which are impracticable in Embodied Vision Navigation. In general, it is nontrivial to calculate the adversarial noise and hack the system at each timestep when the agent is interacting with the environment.

\section{Preliminaries}
In this section, we start with the notations and then introduce the direct application of Universal Adversarial perturbation (UAP)~\cite{moosavi2017universal} on Embodied Vision Navigation.

\subsection{Notations}
\label{section: Notation of UAP}

\paragraph{Markov Decision Process} Following previous works on Embodied Vision Navigation~\cite{chaplot2020object,ye2021hierarchical,gupta2021embodied}, we formulate it as an agent with the policy $\pi$ interacting with the environment, which is a MDP $\mathcal{M} = \{\mathcal{S}, \mathcal{A}, \mathcal{P}, \mathcal{R}, \gamma\}$. Here $\mathcal{S}$ and $\mathcal{A}$ represent the state space and the action space, respectively. At each timestep $t$, the Embodied Vision Navigation agent is at the state $s_t$ and will choose an action $a_t$ from the action space $\mathcal{A}$ via its policy $\pi$. Then the agent will receive an immediate reward $\mathcal{R}(s_t, a_t)$ and move to the next state $s_{t+1}\sim\mathcal{P}(\cdot|s_t, a_t)$. The goal of Embodied Vision Navigation is to find the optimal policy to maximize the expected cumulative reward as
\begin{equation}
    J(\pi) \triangleq \mathbb{E}\left[\sum_{t=0}^{\infty}\gamma^t\mathcal{R}(s_t, a_t)\right],
\end{equation}
where $\gamma$ is the discount factor.

\paragraph{Universal Adversarial perturbation} As widely studied in image classification tasks, UAP~\cite{moosavi2017universal} aims to utilize the same noise $\delta$ to mislead the classifier, which is more computationally efficient and easier to be deployed. In other words, given images obeying the distribution $\mu$ and the victim classifier $f$, the goal is to find the optimal disturbance $\delta^*\in \mathcal{B}_p(0,\epsilon)$ satisfying
\begin{equation}
    \delta^* = \mathop{\arg\min}_{\delta\in\mathcal{B}_p(0,\epsilon)}P_{s\sim\mu}\left\{f(s+\delta)= f(s)\right\},
\end{equation}
where $\mathcal{B}_p(0,\epsilon)$ is the ball centered on $0$ with radius $\epsilon$ in $l_p$-norm corresponding to the feasible region of UAP.

\subsection{UAP in MDP}
\label{sec: baseline uap}

Based on the definition of UAP mentioned in Sec.~\ref{section: Notation of UAP}, we can naively extend UAP to tasks with MDP settings, like Embodied Vision Navigation. Since the policy in MDP might be deterministic or stochastic, we discuss how to apply UAP in these two cases separately. 

When the policy is deterministic, the adversary will collect observations by the victim policy $\pi$ and generate the perturbation $\delta$ with the objective as
\begin{equation}
\label{eq:deter}
    \delta^* = \mathop{\arg\min}_{\delta\in\mathcal{B}_p(0,\epsilon)}P_{s\sim d^\pi(s)}\left\{\pi(s+\delta) =  \pi(s)\right\},
\end{equation}
where $d^\pi(s)$ is the distribution of the state.

When the victim policy is stochastic (e.g., in Proximal Policy Optimization~\cite{schulman2017proximal}), similar to adversarial attack methods in DRL~\cite{Huang2017Adversarial}, we aim at minimizing the probability to take the optimal action, i.e.,
\begin{equation}
\begin{split}
\label{eq:stochastic}
    &\delta^* = \mathop{\arg\min}_{\delta\in\mathcal{B}_p(0,\epsilon)}\expect_{s\sim d^\pi(s)}[\pi(a|s+\delta)],\\
    &\text{where}~a = \mathop{\arg\max}_{a' \in \mathcal{A}} \pi(a'|s).
\end{split}
\end{equation}
In the rest of this paper, we will mainly focus on the stochastic agent~\cite{savva2019habitat} since the deterministic policy is a degenerate case of the stochastic policy, i.e, when the variance of the stochastic agent goes to zero, problem~(\ref{eq:stochastic}) is equivalent to problem~(\ref{eq:deter}).

However, as mentioned in Sec.~\ref{sec-intro}, current UAP does not consider the transition dynamics of the MDP and the sequential dependence of states within a trajectory. We will address those problems in the discussion of our method in Sec.~\ref{method-sec}.

\section{Methodology}
\label{method-sec}
To analyze UAP in MDP, we first introduce $\delta$-MDP to formulate the universal noise's effects. Based on our further analysis of properties of $\delta$-MDP, we propose two novel Consistent Attack methods named Reward UAP and Trajectory UAP.

\subsection{$\delta$-Markov Decision Process}
\label{sec:delta}
To extend UAP to the sequential decision problem, we first formulate it as $\delta$-disturbed Markov Decision Process ($\delta$-MDP), in which the observation in every timestep is disturbed by the same adversarial noise $\delta$.
\begin{mydef}[\textbf{$\delta$-MDP}]
$\delta$-MDP consists of a six-tuple $\mathcal{M}_{\delta} = (\mathcal{S}, \mathcal{A}, \mathcal{R}, \mathcal{P}, \gamma, \delta)$. Here $\mathcal{S}, \mathcal{A}, \mathcal{R}, \mathcal{P}, \gamma$ is the same as MDP and $\delta$ satisfies that $s+\delta\in\mathcal{S}$ holds for $\forall s\in\mathcal{S}$. For any fixed policy $\pi$, when its true state $s\in\mathcal{S}$ is disturbed as $s+\delta$, the agent takes action $a\sim\pi(\cdot|s+\delta)$, arrives at the next state $s'\sim\mathcal{P}(\cdot|s, a)$, and receives the current reward $\mathcal{R}(s, a)$. 
\end{mydef}
For simplicity, we define that $\pi_{\delta}(\cdot|s) = \pi(\cdot|s+\delta)$. It is easy to check that the policy $\pi$ interacting with $\mathcal{M}_{\delta}$ is equivalent to the disturbed policy $\pi_{\delta}$ interacting with $\mathcal{M}$. For any victim policy $\pi$ in $\delta$-MDP, we can define its disturbed cumulative return $J_{\delta}(\pi)$ as
\begin{equation}
    J_{\delta}(\pi) = \mathbb{E}_{s_t, a_t\sim \pi_{\delta}, \mathcal{P}} [ \sum_{t=0}^{\infty}\gamma^t \mathcal{R}(s_t, a_t)].
\end{equation}

The purpose of Consistent Attack for Embodied Vision Navigation is different from that of UAP in image classification, while the latter misleads the neural network to classify most images incorrectly, the former aims to minimize the disturbed cumulative return. To find the best disturbance $\delta^*$ within the disturbance region, our optimization objective is
\begin{equation}
\label{eq-objective}
    \delta^* = \mathop{\arg\min}_{\delta\in\mathcal{B}_p(0,\epsilon)} J_{\delta}(\pi).
\end{equation}

For solving the objective~\ref{eq-objective}, we first analyze some properties of our $\delta$-MDP. Similar to the standard MDP~\cite{sutton2018reinforcement}, we can define the discount future state distribution as $d^{\pi}_{\delta}(s) = (1-\gamma)\sum_{t=0}^{\infty} \gamma^t P_{\delta}(s_t = s)$ and equivalently deform $J_{\delta}(\pi)$ as
\begin{equation}
    J_{\delta}(\pi) = \frac{1}{1-\gamma}\mathbb{E}_{s\sim d^{\pi}_{\delta}(\cdot), a\sim \pi(\cdot|s)}  \mathcal{R}(s, a).
\end{equation}

Also, we can define the $\delta$-disturbed Q function and the $\delta$-disturbed value function as below
\begin{mydef} [\textbf{$\delta$-disturbed Q-function and value function}]
$\delta$-disturbed Q function is defined as
\begin{equation}
    Q_{\delta}^{\pi}(s, a) = \mathbb{E}_{s_t, a_t\sim \pi_{\delta}, \mathcal{P}}\left[ \sum_{t=0}^{\infty}\gamma^t \mathcal{R}(s_t, a_t)\Big{|}s_0=s, a_0=a\right],
\end{equation}
and the $\delta$-disturbed value function is defined as
\begin{equation}
    V_{\delta}^{\pi}(s) = \mathbb{E}_{s_t, a_t\sim \pi_{\delta}, \mathcal{P}}\left[ \sum_{t=0}^{\infty}\gamma^t \mathcal{R}(s_t, a_t)\Big{|}s_0=s\right].
\end{equation}
\end{mydef}
Similar to the Bellman equation in MDP, 
the $\delta$-disturbed Q function and the $\delta$-disturbed value function of our $\delta$-MDP also satisfy the disturbed Bellman equation as below
\begin{thm}[Disturbed Bellman equation for $\delta$ and $\pi$]
\label{thm-1}
$V_{\delta}^{\pi}$ and $Q_{\delta}^{\pi}$ satisfy the disturbed Bellman equation as below
\begin{equation}
\begin{split}
    &V_{\delta}^{\pi}(s) = \mathbb{E}_{a,s'\sim\pi_{\delta},\mathcal{P}}\left[\mathcal{R}(s,a) + \gamma V_{\delta}^{\pi}(s')\right] \\
    &Q_{\delta}^{\pi}(s,a) = \mathcal{R}(s, a) + \gamma \mathbb{E}_{s',a'\sim\mathcal{P},\pi_{\delta}} Q_{\delta}^{\pi}(s',a').
\end{split}
\end{equation}
\end{thm}
We can directly prove it by the existing results in MDP since the policy $\pi_{\delta}(\cdot|s)$ is equivalent to $\pi(\cdot|s+\delta)$.

\subsection{Reward UAP and Trajectory UAP}
\label{section: Reward and Trajectory}
In this part, we will propose two novel Consistent Attack methods named Reward UAP and Trajectory UAP to optimize the objective~\ref{eq-objective}. 

It's difficult to directly minimize $J_{\delta}(\pi)$ since it is always non-convex. Consequently, we consider calculating the gradient $\nabla_{\delta}J_{\delta}(\pi)$ and use gradient descent to optimize it. However, we can  directly calculate $\nabla_{s+\delta}\pi(a|s+\delta)$ via backpropagation rather than $\nabla_{\delta}J_{\delta}(\pi)$. To overcome this gap, by extending the Policy Gradient Theorem~\cite{sutton1999policy}, we propose the Disturbed Policy Gradient Theorem as below
\begin{thm}[Disturbed Policy Gradient]
\label{thm-2}
For any policy $\pi$ and $\delta$-MDP $\mathcal{M}_{\delta}$, we can calculate the gradient of $J_{\delta}(\pi)$ to $\delta$ only by calculating the gradient $\nabla_{s+\delta} \pi(a|s+\delta)$, i.e.,
\begin{equation}
\label{prob:final}
\begin{split}
    &\nabla_{\delta} J_{\delta}(\pi)\\
    =& \frac{1}{1-\gamma}\sum_{s} d_{\delta}^{\pi}(s) \sum_a Q_{\delta}^{\pi}(s, a)\nabla_{\delta} \pi(a|s+\delta)\\
    =& \frac{1}{1-\gamma}\mathbb{E}_{s\sim d_{\delta}^{\pi}} \mathbb{E}_{a\sim\pi(\cdot|s+\delta)} \left[Q_{\delta}^{\pi}(s, a)\nabla_{s+\delta} \left[\log\pi(a|s+\delta)\right]\right].
\end{split}
\end{equation}
\end{thm}                                               
\begin{proof}
Based on Theorem~\ref{thm-1}, we can first build the connection of $\nabla_{\delta} V_{\delta}^{\pi}(s)$ and $\nabla_{\delta}\pi(a|s+\delta)$ as below
\begin{equation}
\begin{split}
    \nabla_{\delta}V_{\delta}^{\pi}(s)
    =& \nabla_{\delta}\left[\sum_a \pi(a|s+\delta)Q_{\delta}^{\pi}(s,a)\right]\\
    =& \sum_a Q_{\delta}^{\pi}(s,a) \nabla_{\delta}\pi(a|s+\delta)\\
    +& \sum_a \pi(a|s+\delta) \nabla_{\delta} Q_{\delta}^{\pi}(s,a)\\
    =& \sum_a Q_{\delta}^{\pi}(s,a) \nabla_{\delta}\pi(a|s+\delta)\\
    +& \sum_a \pi(a|s+\delta) \gamma\sum_{s'}P(s'|s,a)\nabla_{\delta}V_{\delta}^{\pi}(s').
\end{split}
\end{equation}
Similar to the property of $d^{\pi}$~\cite{ying2022towards}, we can prove that $d_{\delta}^{\pi}(s) - (1-\gamma)\mathcal{P}(s_0=s)=\gamma\sum_{s'}d_{\delta}^{\pi}(s')\sum_{a}\pi(a|s'+\delta)\mathcal{P}(s|s', a)$. This property extends the discount future state distribution $d_{\delta}^{\pi}(s)$ in one step and builds the connection between it and other discount future state distribution $d_{\delta}^{\pi}(s')$.
Based on this property, we can further prove that
\begin{equation}
\begin{split}
    &\sum_s d_{\delta}^{\pi}(s) \nabla_{\delta} V_{\delta}^{\pi}(s)\\ 
    =& \sum_s d_{\delta}^{\pi}(s)\sum_a Q_{\delta}^{\pi}(s,a) \nabla_{\delta}\pi(a|s+\delta)\\
    +& \sum_s d_{\delta}^{\pi}(s)\sum_a \pi(a|s+\delta) \gamma\sum_{s'}P(s'|s,a)\nabla_{\delta}V_{\delta}^{\pi}(s')\\
    =&  \sum_s d_{\delta}^{\pi}(s) \sum_a Q_{\delta}^{\pi}(s,a) \nabla_{\delta}\pi(a|s+\delta)\\
    +& \sum_{s'}\left[d_{\delta}^{\pi}(s') - (1-\gamma)P(s_0=s')\right]\nabla_{\delta}V_{\delta}^{\pi}(s').
\end{split}
\end{equation}
Thus we have
\begin{equation}
\begin{split}
    &\sum_{s'} (1-\gamma)P(s_0=s') \nabla_{\delta}V_{\delta}^{\pi}(s')\\ 
    =&  \sum_s d_{\delta}^{\pi}(s) \sum_a Q_{\delta}^{\pi}(s,a) \nabla_{\delta}\pi(a|s+\delta).
\end{split}
\end{equation}
Consequently, we can get
\begin{equation}
\begin{split}
    &\nabla_{\delta} J_{\delta}(\pi)\\
    =& \sum_{s'} P(s_0=s')\nabla_{\delta}V_{\delta}^{\pi}(s')\\
    =& \frac{1}{1-\gamma}\sum_s d_{\delta}^{\pi}(s) \sum_a Q_{\delta}^{\pi}(s,a) \nabla_{\delta}\pi(a|s+\delta)\\
    =& \frac{1}{1-\gamma}\mathbb{E}_{s\sim d_{\delta}^{\pi}} \mathbb{E}_{a\sim\pi(\cdot|s+\delta)} \left[Q_{\delta}^{\pi}(s, a)\nabla_{s+\delta} \left[\log\pi(a|s+\delta)\right]\right],
\end{split}
\end{equation}
here it is well known that $\nabla_{\delta}[\log \pi(a|s+\delta)] = \nabla_{s+\delta}[\log \pi(a|s+\delta)]$ by the chain rule. Therefore, we have proven this result.
\end{proof}
As shown in Theorem~\ref{thm-2}, the gradient $\nabla_{\delta} J_{\delta}(\pi)$ is based on $\nabla_{\delta} \left[\log\pi(a|s+\delta)\right]$ for each state-action pair $(s,a)$ sampled from the distribution. However, the weight varies among different $(s,a)$ and depends on the disturbed Q function $Q_{\delta}^{\pi}(s_t, a_t)$. This result also theoretically shows that directly applying UAP to MDP, i.e., calculating the gradient of different $(s,a)$ and averaging them, ignores the dynamic connection of different $(s, a)$ and thus is not optimal.

Based on Theorem~\ref{thm-2}, we can naturally calculate the optimal universal perturbation $\delta^*$ to minimize $J_{\delta}(\pi)$ via gradient descent. However, there are still two critical problems for calculating $\nabla_{\delta} J_{\delta}(\pi)$ in practice. i.e., how to sample state-action pair $(s, a)$ from the distribution $d_{\delta}^{\pi}(s)\pi(a|s+\delta)$ and how to estimate the disturbed Q function. To address those problems, we propose two novel Consistent Attack methods: Reward UAP and Trajectory UAP.


\begin{algorithm}[tb]
\caption{Consistent Attack}
\label{al:consistent-uap}
Given the victim policy $\pi$, the disturbance range $\epsilon$, the upper bound on the number of samples $m=nl$, and the learning rate $\alpha$.
\begin{algorithmic}[1] 
    \STATE Initialize universal perturbation $\delta_0^{(0)} = 0$ ($\delta_i^{(j)}$ indicates $\delta_i$ in Reward Attack and $\delta_i^j$ in Trajectory Attack)
    
    \FOR {$k = 0, 1,2,..., n-1$}
    \STATE Random collect a set of $l$ trajectories $\mathcal{D} = \{\tau_i\}_{i=1}^l$ by the disturbed policy \bm{${\pi_{\delta_k}}$}, here trajectory $\tau_i = \{s_1^{i},a_1^{i},r_1^{i},s_2^{i},...\}$ with $g_i$ indicates whether the agent reaches the goal.
    \STATE \textcolor{blue}{\textbf{Option 1: Gradient Update on Reward UAP}\{}
    \STATE \quad Compute the reward-to-go as $\hat R_t^i = \sum_{t' = t}^{T}\gamma^{t'-t}r_{t'}^i$ or use the Q function.
    \STATE \quad Update the disturbance:
    \begin{equation*}
        \delta_{k+1} = \delta_k - \alpha \sum_i\sum_{t=0}^T \textcolor{blue}{\hat R_t^i} \nabla_{\delta_k}\left[\log\pi_{\delta_k}(a_t^{i}|s_t^{i})\right].
    \end{equation*}
    \textcolor{blue}{\}}
    \STATE \textcolor{red}{\textbf{Option 2: Gradient Update on Trajectory UAP}\{}
    \begin{ALC@g}
    \STATE \quad Update the disturbance:
    \begin{equation*}
        \delta_{k+1} = \delta_k - \alpha \sum_i \textcolor{red}{g_i}\sum_{t=0}^T \gamma^{T-t}\nabla_{\delta_k}\left[\log\pi_{\delta_k}(a_t^{i}|s_t^{i})\right].
    \end{equation*}
    \end{ALC@g}
    \textcolor{red}{\}}
    \ENDFOR
    \STATE Project the universal perturbation to the boundary of the disturbance region:
    \begin{equation*}
        \delta_{n} = \epsilon \frac{\delta_{n}}{\|\delta_{n}\|}.
    \end{equation*}
\end{algorithmic}
\end{algorithm}

\textbf{Reward UAP.} First, we propose Reward UAP to handle those two critical problems in UAP mentioned above. 

For sampling state-action pairs from the disturbed distribution, we propose a multi-step optimization strategy. Specifically, at each optimization step $k$, we have current adversarial noise $\delta_k$ and sample a set of trajectories via $\pi_{\delta_k}$. Moreover, we use state-action pairs from these trajectories to optimize $\delta_k$ by using Theorem~\ref{thm-2}.

For estimating the disturbed Q function which is difficult to directly calculate, we propose two kinds of surrogates. First, we can use the return of the trajectory collected by $\pi_{\delta}$ to estimate $Q_{\delta}^{\pi}(s_t, a_t)$, i.e.,
\begin{equation}
    \hat{R}_t \triangleq \sum_{t'=t}^{\infty}\gamma^{t'-t}\mathcal{R}(s_{t'},a_{t'})\approx Q_{\delta}^{\pi}(s_t, a_t).
\end{equation}

Also, we can directly use the victim's Q function to approximate the disturbed Q function, i.e., $\hat{R}_t \triangleq Q^{\pi}(s_t, a_t) \approx Q_{\delta}^{\pi}(s_t, a_t)$. Consequently, in our Reward UAP, at each epoch $k$, we can estimate $\nabla_{\delta_k} J_{\delta_k}(\pi)$ via
\begin{equation}
\begin{split}
    &\nabla_{\delta_k} J_{\delta_k}(\pi)
    \approx \frac{1}{m}\sum_{i=1}^m\sum_{t} \hat{R}_t^i \nabla_{s_t^i+\delta_k}[\log \pi(a_t^i|s_t^i+\delta_k)],
\end{split}
\end{equation}
and use this approximator to update $\delta_k$. The pseudo-code of Reward UAP is in Algorithm~\ref{al:consistent-uap}.



\textbf{Trajectory UAP.} 
Although Reward UAP can find the optimal gradient direction, which is the best direction for reducing the disturbed cumulative return, there are still some practical problems in Reward UAP. First, by Theorem~\ref{thm-2}, we need to estimate $Q_{\delta}^{\pi}(s, a)$ to calculate the gradient $\nabla_{\delta} J_{\delta}(\pi)$. There are two common ways for estimating $Q_{\delta}^{\pi}(s, a)$, i.e., use the return of trajectories to directly estimate it or train a neural network to estimate it. However, the former one may suffer from high variance and the latter one needs to re-train the neural network whenever the noise $\delta$ is changed. Second, in some practical settings like the PointGoal task in Embodied Vision Navigation, the adversary can only determine whether the agent has reached the goal rather than having access to the actual reward signals, which are related to the distance between the agent and the goal in Habitat~\cite{savva2019habitat}.

To handle those challenges, we regard Embodied Vision Navigation as a goal-condition MDP, of which the return solely depends on whether the agent has reached the goal $s_g$, i.e.,
\begin{equation}
\mathcal{R}_{g}(s_t, a_t) = \left\{ \begin{aligned}
1   &\quad ~\text{the agent reaches the goal, i.e. $s_t=s_g$} \\
0   &\quad ~\text{otherwise}. \\
\end{aligned} \right.
\end{equation}

This situation is more practical since we can easily tell whether the agent has reached the goal, and the estimation of the disturbed Q function is more stable. In this case, we can directly calculate the disturbed Q function of state-action pair $(s_t, a_t)$ in a trajectory $(s_0, a_0, s_1, a_1, ..., s_T)$ as
\begin{equation}
Q_{g}(s_t, a_t) = \left\{ \begin{aligned}
\gamma^{T-t}   &\quad ~\text{if $s_T=s_g$} \\
0   &\quad ~\text{otherwise}. \\
\end{aligned} \right.
\end{equation}
Moreover, we introduce $g$ to indicate whether the agent reaches the goal at this trajectory, i.e., $g = 1$ if $s_T=s_g$ and $g = 0$ otherwise. We have $Q_{g}(s_t, a_t) = \gamma^{T-t}g$. Based on the above analyses, we propose Trajectory UAP, of which the pseudo-code is in algorithm~\ref{al:consistent-uap}. In summary, we use Trajectory UAP when we can only know whether the agent successfully reaches the goal and Trajectory UAP is much more stable.

\begin{table*}[h]
\renewcommand\tabcolsep{4.5pt}  
\begin{center}
    \begin{tabular}{ccc|c|ccc|ccc|ccc}
    \toprule
    \multicolumn{4}{c}{Environments} &
    \multicolumn{3}{c}{Habitat-test} &
    \multicolumn{3}{c}{Gibson} & \multicolumn{3}{c}{MP3D}\\
    \hline
    Baseline & Sensors & Adversary& $\eta$ & Rew & Succ & SPL & Rew & Succ & SPL & Rew & Succ & SPL\\
    \midrule
    \multirow{8}*{RL(PPO)} & \multirow{4}*{RGB}& None & - & -0.23 & 0.52 & 0.41 & 4.96 & 0.83 & 0.68 & 4.42 & 0.51 & 0.38 \\ 
    & & UAP & 0.5 & -2.13 & 0.13 & 0.04 & -4.98 & 0.05 & 0.02 & 0.24 & \textbf{0.00} & \textbf{0.00} \\
    & & Reward UAP & 0.5 & -1.20 & 0.16 & 0.07 & \textbf{-5.24} & 0.02 & 0.02 & \textbf{-2.54} & 0.01 & 0.01 \\
    &  & Trajectory UAP & 0.5 & \textbf{-6.64} & \textbf{0.05} & \textbf{0.01} & -1.63 & \textbf{0.01} & \textbf{0.00} & -0.54 & 0.02 & 0.02 \\
    \cline{2-13}
    & \multirow{4}*{Depth} & None & -& 6.73 & 0.93 & 0.81 & 6.63 & 0.92 & 0.84 & 9.53 & 0.80 & 0.67 \\ 
    & & UAP & 0.03 & -0.34 & \textbf{0.48} & 0.39 & 0.09 & 0.44 & 0.30 & -3.74 & 0.04 & 0.03 \\
    &  & Reward UAP & 0.03 & \textbf{-0.75} & 0.49 & \textbf{0.38} & -3.78 & 0.09 & 0.05 & -4.43 & 0.02 & 0.02 \\
    &  & Trajectory UAP & 0.03 & -0.63 & 0.49 & 0.39 & \textbf{-4.23} & \textbf{0.05} & \textbf{0.02} & \textbf{-4.57} & \textbf{0.01} & \textbf{0.01} \\
    \bottomrule
    \end{tabular}
    \end{center}\vspace{-0.3cm}
    \label{table_1}
    \caption{Performance per trajectory of the attacked policy on the PointGoal task tested in the Habitat-test, Gibson, and MP3D datasets under RGB and Depth sensor. We report the average reward, episode success rate, and SPL, where the lower reward, episode success rate, and SPL indicate a stronger adversary.
    }
    \vspace{-0.3cm}
\end{table*}

\section{Experimental Results}
In this section, we conduct extensive experiments to show the effectiveness of our Consistent Attack methods, including Reward UAP and Trajectory UAP, for Embodied Vision Navigation agents with different sensors in Habitat scenes~\cite{savva2019habitat}. 

\begin{figure*}[t]
    \centering
    \subfigure[\centering Victim Policy: 
    Succ = 1.0,\newline 
    SPL = 0.71,
    Reward = 4.28]{
        \includegraphics[width = 0.22\linewidth]{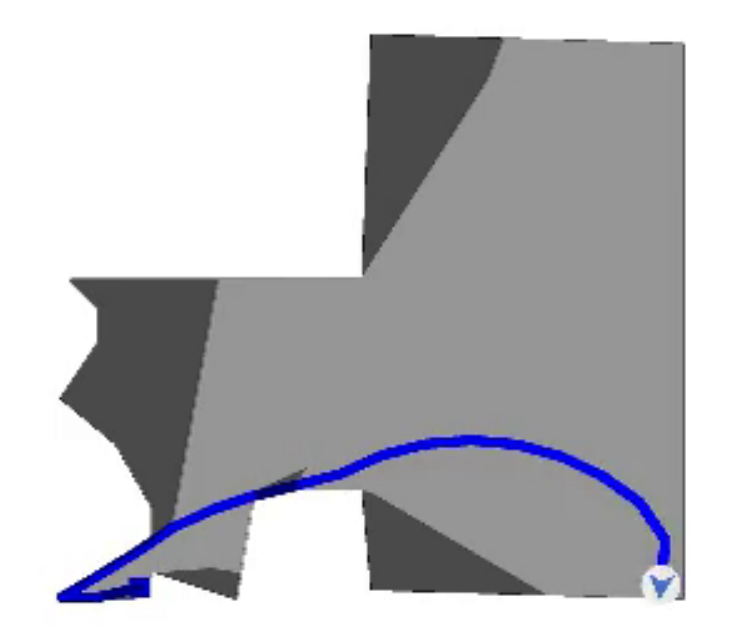}
    }
    \subfigure[\centering Attacked Policy:
    Succ = 0.0,\newline
    SPL = 0.0,
    Reward = -4.78]{
        \includegraphics[width = 0.22\linewidth]{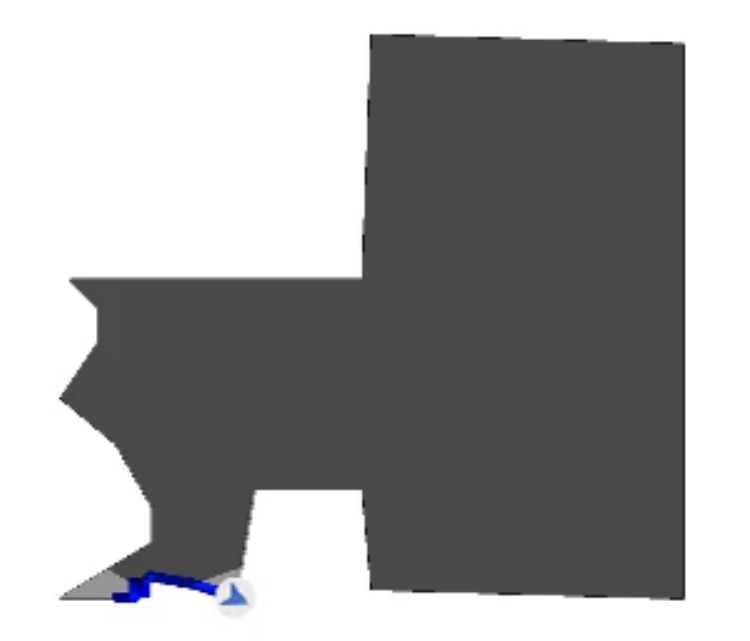}
    }
    \subfigure[\centering Victim Policy:
    Succ = 1.0,\newline
    SPL = 0.31,
    Reward = 3.35]{
        \includegraphics[width = 0.22\linewidth]{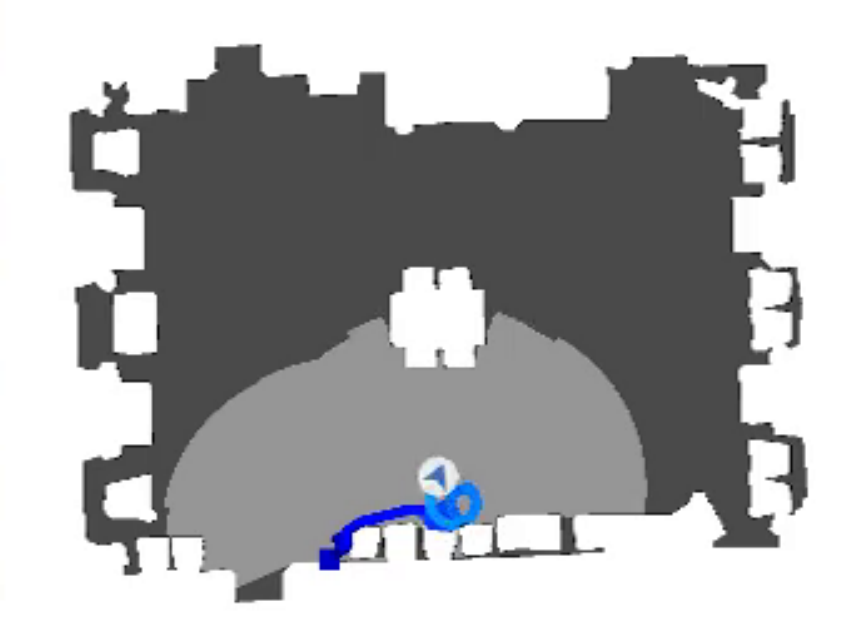}
    }
    \subfigure[\centering Attacked Policy:
    Succ = 0.0,\newline 
    SPL = 0.0,
    Reward = -14.38]{
        \includegraphics[width = 0.22\linewidth]{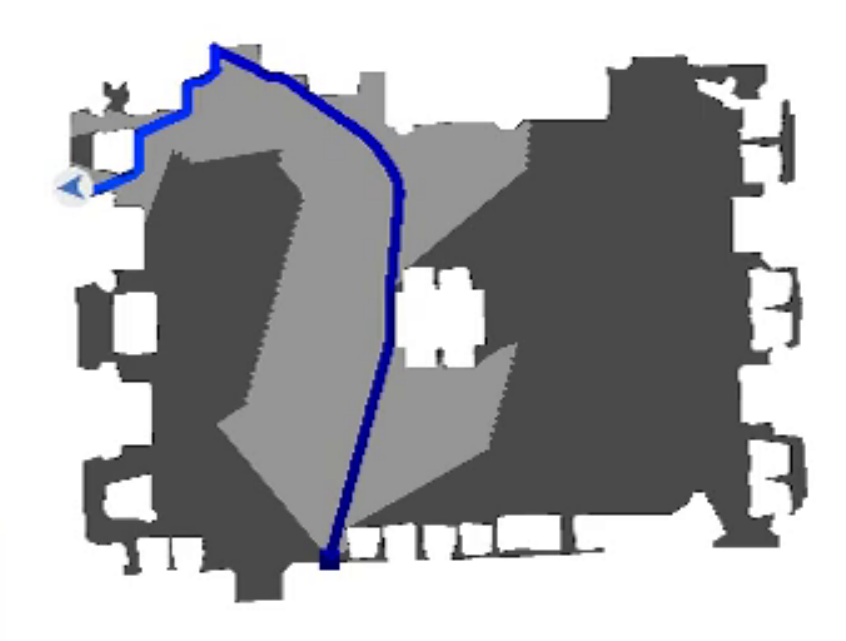}
    }
    \vspace{-0.3cm}
    \caption{Visualization of trajectories sampled by the victim policy and the attacked policy via Reward UAP on the PointGoal task in the Habitat-test dataset. The \textcolor{blue}{blue line} is the trajectory of the Embodied agent and we also report the trajectory's Succ, SPL, and Reward.}
    \label{fig:n}
    \vspace{-0.3cm}
\end{figure*}


\subsection{Environment setups}
\textbf{Environments and Tasks.}
For evaluating the robustness of agents in Embodied Vision Navigation tasks, we choose the popular platform Habitat~\cite{savva2019habitat} with the PointGoal task~\cite{anderson2018evaluation}. Following previous work~\cite{savva2019habitat}, our experiments are based on the Habitat-test dataset~\cite{savva2019habitat}, the Gibson dataset~\cite{xia2018gibson}, and the MatterportP3D (MP3D) dataset~\cite{chang2017matterport3d}. 

\textbf{Victim Policies.} To investigate the vulnerability of Embodied Vision Navigation, we choose the stochastic---as mentioned in Sec.~\ref{sec: baseline uap}---victim policies~\cite{savva2019habitat} based on Proximal Policy Optimization (PPO)~\cite{schulman2017proximal}, which is a classic reinforcement learning algorithm, and different vision sensors like RGB input (RGB) and depth input (Depth). The details of the victim policy can be found in habitat-baseline~\cite{savva2019habitat}. All these agents are officially released pre-trained models~\cite{savva2019habitat} with competitive performance in the PointGoal task. 

\textbf{Baselines.}
We compare our Consistent Attack, including Trajectory UAP and Reward UAP, with standard UAP method~\cite{moosavi2017universal} mentioned at Sec.~\ref{sec: baseline uap}, which is a direct extension of UAP to Embodied Vision Navigation by treating all input images mutually independent. As shown in Eq.~\ref{eq:stochastic}, standard UAP for image classification aims to find a universal disturbance to mislead the classifier on most input images. For the implementation of standard UAP, we sample trajectories with the victim policy and use all observations as the dataset for calculating the disturbance.


\begin{table}[h]
   \centering
    \begin{tabular}{cccccc}
    \toprule
    Adversary & $m$ & $\eta$ & Rew & Succ & SPL\\
    \midrule
    None &-&-& 6.73 & 0.93 & 0.81 \\
    \hline
    UAP      &  5 & 0.03 & -0.34 & \textbf{0.48} & \textbf{0.39} \\
    Reward UAP & 5 & 0.03 & \textbf{-0.75} & \textbf{0.49} & \textbf{0.38} \\
    Trajectory UAP & 5 & 0.03 & -0.63 & \textbf{0.49} & \textbf{0.39} \\
    \hline
    UAP & 10 & 0.03 & -0.50 & 0.51 & 0.38 \\
    Reward UAP & 10 & 0.03 & \textbf{-1.57} & \textbf{0.43} & \textbf{0.28} \\
    Trajectory UAP & 10 & 0.03 & -1.22 & 0.45 & 0.32 \\    
    \hline
    UAP & 15 & 0.03 & -0.60 & 0.49 & 0.36 \\
    Reward UAP & 15 & 0.03 & \textbf{-4.15} & \textbf{0.16} & \textbf{0.12} \\
    Trajectory UAP & 15 & 0.03 & -1.11 & 0.47 & 0.32 \\    
    \bottomrule
    \end{tabular}
    \caption{Performance of the adversary with different numbers of training trajectories $m$ tested on the Habitat-test dataset with depth model. We report the average of reward, episode success rate, and SPL. For each $m$, numbers within 5 percent of the maximum in every individual environment are \textbf{bold}.}
    \label{table_2}
    \vspace{-0.4cm}
\end{table}

\textbf{Evaluation.}
To show the effectiveness of the Consistent Attack methods, we compare our methods with UAP. For a fair comparison, we choose a similar setup in classical UAP~\cite{moosavi2017universal}. In particular, we evaluate the perturbation in $l_2$-norm and set $\eta = \frac{\epsilon}{\sqrt{d}} = 0.5$ and 0.03 for RGB sensor and depth sensor respectively, following the traditional setting in classical UAP works~\cite{moosavi2017universal}. Besides, we evaluate the attacked agent at every task separately, e.g., $100$ episodes in the Habitat-test dataset, $994$ episodes in the Gibson dataset, and $495$ episodes in the Matterport3D dataset. At each episode, the agent takes up to $500$ actions to find the navigable path following the setting in habitat~\cite{savva2019habitat}. For each adversarial attack method, we first evaluate the average reward of the attacked policy, and we also measure the attacked policy's performance via the Episode Success Rate (Succ) as well as the Success weighted by Path Length (SPL). 




\subsection{Results}
\label{sec: Results}

\textbf{Evaluation of Consistent Attack.}  
For all three tasks (Habitat-test, Gibson, MP3D) and the victim PPO model with different sensors (RGB, Depth), We report the victim policy's performance and its performance under different adversaries (UAP, Reward UAP, Trajectory UAP). Here all attack methods only sample 5 trajectories to calculate noises. For each experiment, we report it under three different metrics (reward, episode success rate, SPL). As shown in Table.~\ref{table_1}, all these adversarial attack methods can reduce the performance in all tasks, especially in more complicated tasks like Gibson. This observation indicates that current Embodied Vision Navigation models are vulnerable to adversarial noises. Moreover, compared to UAP, our consistent attack methods, including Reward UAP and Trajectory UAP, can significantly reduce the victim's performance under all three metrics because we consider the system dynamics and thus optimize a better objective of the noises.

In almost all tasks tested on the Gibson and MP3D datasets under RGB and Depth sensors, Trajectory UAP outperforms UAP and Reward UAP in reward, Succ, and SPL. It matches our analyses that using goal signal to estimate the disturbed Q function is more stable when the trajectory number for calculating the noise is limited in Embodied Vision Navigation task. Besides, Trajectory UAP almost decreases the Episode Success Rate to zero in MP3D-Depth, which shows that there exist serious potential risks for applying Embodied Vision Navigation methods to the real world. In Fig.~\ref{fig:n}, we plot some trajectories, with their Succ, SPL, and Reward, of the victim policy and the attacked policy. As shown here, our Consistent Attack can significantly reduce the reward of the victim policy and mislead the Embodied agent to a place far away from the goal.

\textbf{Ablation Study of the Sampling Number $m$.} In this part, we analyze how the number of trajectories $m$ used for calculating noises affects different algorithms, including UAP, Reward UAP, and Trajectory UAP. Due to the limit of computational costs, we evaluate our methods in the Habitat-test~\cite{savva2019habitat} dataset with the depth sensor. As shown in Table.~\ref{table_2}, when the number of training trajectories $m$ increases, Reward UAP's attack capacity significantly improves and is the strongest adversarial attack under all metrics. It indicates that with the increment of sample size, reward UAP can get a more accurate estimation of the gradient $\nabla_{\delta} J_{\delta}$ and thus become a stronger adversary.

\section{Conclusion}
In this work, for evaluating the robustness of existing Embodied Vision Navigation methods, we consider a reasonable setting where the adversary adds a constant perturbation to every observation. We first extend UAP framework under MDP setting and propose $\delta$-MDP to formalize this problem. Based on analyzing properties of $\delta$-MDP, we propose two novel Consistent Attack methods, named Reward UAP and Trajectory UAP, by considering the disturbed state-action distribution and the disturbed Q function. In experiments, we effectively and efficiently mislead the agents in Habitat, which indicates the potential risk to apply these embodied agents to the real world. We hope our study can encourage more future work for designing more robust and reliable Embodied Vision Navigation agents.



\bibliography{main}


\end{document}